%% file: main.tex
\documentclass{article}
\usepackage{kr}

\usepackage{times}
\usepackage{soul}
\usepackage{url}
\usepackage[hidelinks]{hyperref}
\usepackage[utf8]{inputenc}
\usepackage[small]{caption}
\usepackage{graphicx}
\usepackage{amsmath}
\usepackage{amsthm}
\usepackage{booktabs}
\usepackage{algorithm}
\usepackage{algorithmic}

\usepackage{xspace}
\usepackage{amssymb}
\usepackage{tikz}
\usepackage{multirow}
\usepackage{placeins}
\usepackage{enumitem}

\usepackage[most]{tcolorbox}

\newtheorem{theorem}{Theorem}

\newtheorem{Definition}{Definition}
\newtheorem{corollary}{Corollary}

\title{From Subsumption to Satisfiability: LLM-Assisted Active Learning for OWL Ontologies}

\author{%
Haoruo Zhao$^1$\and
Wenshuo Tang$^1$\and
Duncan Guthrie$^1$\and
Michele Sevegnani$^1$\and
David Flynn $^1$\and
Paul Harvey $^1$ \\
\affiliations
$^1$University of Glasgow\\
\emails
\{Haoruo.Zhao, Wenshuo.Tang\}@glasgow.ac.uk,
2266457G@student.gla.ac.uk,
\{Michele.Sevegnani, David.Flynn,Paul.Harvey\}@glasgow.ac.uk
}

\begin{document}

\maketitle

\input{macro}

\begin{abstract}

In active learning, membership queries (MQs) allow a learner to pose questions to a teacher, such as ``Is every apple a fruit?”, to which the teacher responds correctly with ``yes” or ``no”. These MQs can be viewed as subsumption tests with respect to the target ontology. Inspired by the standard reduction of subsumption to satisfiability in description logics, we reformulate each candidate axiom into its corresponding counter-concept and verbalise it in controlled natural language before presenting it to Large Language Models (LLMs). We introduce LLMs as a third component that provides real-world examples approximating an instance of the counter-concept. This design property ensures that only Type II errors may occur in ontology modelling; in the worst case, these errors merely delay the construction process without introducing inconsistencies. Experimental results on 13 commercial LLMs show that recall, corresponding to Type II errors in our framework, remains stable across several well-established ontologies.

\end{abstract}

\input{Introduction}

\input{BackgroundKnowledge}

\input{Method}
\input{Evaluation}
\input{Conclusion}



\FloatBarrier

\bibliographystyle{kr}
\bibliography{Reference/order,Reference/biblio-clean,Reference/short-string,Reference/background-extra}

\end{document}

%% file: macro.tex

\newtheorem{claim}{Claim}

\newcommand{\Ont}{\mathcal{O}}
\newcommand{\OntN}[1]{\mathcal{O}_{#1}}
\newcommand{\AtomN}[1]{\mathfrak{a}_{#1}}
\newcommand{\Atom}[1]{\mathfrak{#1}}
\newcommand{\axiom}{\alpha}
\newcommand{\sigaxiom}{\widetilde{\alpha}}

\newcommand{\NExptime}{\textsc{NExpTime}}

\newcommand{\R}{\ensuremath{\mathcal{R}}\xspace}
\newcommand{\ConceptName}{\textbf{N}_\textbf{C}}
\newcommand{\RoleName}{\textbf{N}_\textbf{R}}
\newcommand{\ConceptNameA}[1]{\mathsf{CanS}(#1)}
\newcommand{\CandiC}{\mathsf{CanS}}

\newcommand{\Inv}{\mathop{\mathsf{Inv}}}
\newcommand{\Irr}{\mathop{\mathsf{Irref}}}

\newcommand{\Refl}{\mathop{\mathsf{Ref}}}

\newcommand{\Sym}{\mathop{\mathsf{Sym}}}
\newcommand{\Asym}{\mathop{\mathsf{Asym}}}
\newcommand{\Dis}{\mathop{\mathsf{Disj}}}

\newcommand{\Protege}{Prot\'{e}g\'{e}\xspace} 

\newcommand{\ConOnt}[1]{\mathsf{C}(#1)}

\newcommand{\At}[1]{\mathsf{Ats}(#1)}
\newcommand{\Atwithout}[1]{\mathsf{At}(#1)}
\newcommand{\Ats}[2]{$\mathsf{Ats}(#1, {#2})$\xspace}

\newcommand{\MinAt}[1]{\mathsf{MinAts}(#1)}
\newcommand{\MinAtwithout}[1]{\mathsf{MinAts}(#1)}

\newcommand{\SigOnt}{\widetilde{\mathcal{O}}}
\newcommand{\SigModule}[2]{\widetilde{\mathcal{M}}( #1, #2) }
\newcommand{\SigM}{\widetilde{\mathcal{M}}}
\newcommand{\moduleSig}[1]{\widetilde{\module_{#1}}}
\newcommand{\module}{\mathcal{M}}
\newcommand{\Module}[2]{\mathcal{M}( #1, #2)}
\newcommand{\MOverhead}{\module_3}

\newcommand{\MA}{\mathcal{M}_{A}}
\newcommand{\SigMA}{\widetilde{\mathcal{M}_{A}}}
\newcommand{\SigPricinpalAtom}{\widetilde{\downarrow\!\mathfrak{a}}}
\newcommand{\SigPricinpalatom}[1]{\widetilde{\downarrow\!\mathfrak{a}_{#1}}}
\newcommand{\SigPricinpalatomL}[1]{\widetilde{\downarrow\!\mathfrak{#1}}}
\newcommand{\PricinpalAtom}{\mathop{\downarrow\!\mathfrak{a}}}
\newcommand{\Pricinpalatom}[1]{\mathop{\downarrow\!\mathfrak{a}_{#1}}}
\newcommand{\PricinpalatomL}[1]{\mathop{\downarrow\!\mathfrak{#1}}}

\newcommand{\PricinpalAtomN}[1]{\mathord{\downarrow}\mathfrak{a}_{#1}}

\newcommand{\PBFBotPolarityFormWithout}[2]{\ensuremath{\varphi^{\bot}_{#1} ({#2})}\xspace}
\newcommand{\PBFTopPolarityFormWithout}[2]{\ensuremath{\varphi^{\top}_{#1} ({#2})}\xspace} 

\newcommand{\SigAtom}{\widetilde{\mathfrak{a}}}
\newcommand{\Sigatom}[1]{\widetilde{\mathfrak{#1}}}
\newcommand{\SigatomN}[1]{\widetilde{\mathfrak{a}_{#1}}}

\newcommand*{\circled}[1]{\lower.7ex\hbox{\tikz\draw (0pt, 0pt)%
    circle (.5em) node {\makebox[1em][c]{\small #1}};}}

\newcommand{\ConS}{\textbf{Con}(\Sigma)\xspace}

\newcommand{\AD}{\mathfrak{A}(\mathcal{O})}
\newcommand{\ADOther}[1]{\mathfrak{A}{#1}}
\newcommand{\ADN}[1]{\mathfrak{A}(\mathcal{O}_{#1})}

\renewcommand{\mod}[1]{\ensuremath{#1\textup{-mod}}} 
\newcommand{\PBFDletaAxiomSigWithout}{\ensuremath{\varphi^{\delta}_{\Sigma}(\alpha)}\xspace}

\newcommand{\HAtoms}{\textbf{HA}(\Ont)}
\newcommand{\BTop}{\mathsf{BTop}(\Ont)}
\newcommand{\ConceptSet}{\mathsf{ConSet}}
\newcommand{\CheckedAtoms}{\mathsf{CheckedAtoms}}

\newcommand{\SetModule}{\mathcal{S}_{\mathcal{M}}}

\newcommand{\HighestAtoms}[1]{\mathsf{HAtoms}(#1)}
\newcommand{\LowestAtoms}[1]{\mathsf{LAtoms}(#1)}

\newcommand{\scname}[1]{\textsc{#1}}
\newcommand{\klone}{\scname{Kl-One}}
\newcommand{\Classic}{\scname{Classic}}
\newcommand{\Kris}{\scname{Kris}}

\newcommand{\DLL}{\mathcal{L}}
\newcommand{\dlel}{\mathcal{EL}}
\newcommand{\dlalc}{\mathcal{ALC}}

\newcommand{\dlelplus}{\ensuremath{\mathcal{EL}^{{+}{+}}}}
\newcommand{\cardinalitySigOnt}{\#\widetilde{\ont}}
\newcommand{\cardinalityOnt}{\#\ont}

\newcommand{\PBFDeltaEvaluationAtom}[1]{\varphi^{\delta}_{\Sigma}(#1)}

\newcommand{\PBFDeltaConcept}[1]{\varphi^{\delta}(#1)}

\newcommand{\PBFDeltaPositiveEvaluationC}[1]{\varphi^{\delta}_{+,\Sigma}(#1)}
\newcommand{\PBFBotPositiveEvaluationC}[1]{\varphi^{\bot}_{+,\Sigma}(#1)}
\newcommand{\PBFTopPositiveEvaluationC}[1]{\varphi^{\top}_{+,\Sigma}(#1)}

\newcommand{\PBFDeltaNegativeEvaluationC}[1]{\varphi^{\delta}_{-,\Sigma}(#1)}
\newcommand{\PBFBotNegativeEvaluationC}[1]{\varphi^{\bot}_{-,\Sigma}(#1)}
\newcommand{\PBFTopNegativeEvaluationC}[1]{\varphi^{\top}_{-,\Sigma}(#1)}

\newcommand{\sroiq}{\ensuremath{\mathcal{SROIQ}}\xspace}
\newcommand{\shiq}{\ensuremath{\mathcal{SHIQ}}\xspace}
\newcommand{\shiqD}{\ensuremath{\mathcal{ALCIQ(D)}}\xspace}

\newcommand{\Tr}{\mathop{\mathsf{Trans}}}
\newcommand{\Fn}{\mathop{\mathsf{Func}}}

\newcommand{\interpretation}{\mathcal{I}}
\newcommand{\Interpretation}[1]{\mathcal{#1}}

\newcommand{\Res}{\mathsf{Res}}

\newcommand{\MORe}{\mathsf{MORe}}
\newcommand{\Chainsaw}{\mathsf{Chainsaw}}
\newcommand{\HermiT}{\mathsf{HermiT}}
\newcommand{\Pellet}{\mathsf{Pellet}}
\newcommand{\ELK}{\mathsf{ELK}}
\newcommand{\Factplus}{\mathsf{FaCT++}}
\newcommand{\JFact}{\mathsf{JFact}}
\newcommand{\Fact}{\mathsf{Fact}}
\newcommand{\Racer}{\scname{Racer}}
\newcommand{\CEL}{\scname{CEL}}
\newcommand{\Konclude}{\mathsf{Konclude}}
\newcommand{\Sequoia}{\mathsf{Sequoia}}
\newcommand{\CB}{\mathsf{CB}}
\newcommand{\Cr}{\mathsf{Crane}}

\newcommand{\ReAD}{\mathsf{ReAD}}
\newcommand{\ReADFeature}[1]{\mathsf{ReAD_{#1}}}
\newcommand{\Crane}{\mathsf{ReAD_{Crane}}}

\newcommand{\sigL}{\Sigma_{\mathsf{EL}}}
\newcommand{\ontL}{\Ont_{\mathsf{EL}}}
\newcommand{\ontRe}{\Ont_{\mathsf{RAs}}}

\newcommand{\sameAtom}{\mathsf{inSameAtom}}
\newcommand{\restCheck}{\mathsf{restCheckList}}

\newcommand{\MEL}{\module_{\mathsf{EL}}}
\newcommand{\SigMEL}{\widetilde{\module_{\mathsf{EL}}}}
\newcommand{\MRAs}{\module_{\mathsf{RAs}}}
\newcommand{\sigForModule}[1]{\widetilde{#1}}

\newcommand{\STs}[1]{\#ST_{#1}}


\newcommand{\sat}[1]{\textbf{Subs2}(#1)}
\newcommand{\sub}[2]{\mathsf{Subs}(#1, #2)}
\newcommand{\subuni}[1]{\mathsf{Subs}(#1)}

\newcommand{\TEL}{\mathcal{T}_{\mathsf{EL}}}
\newcommand{\TBoxOnly}{\mathcal{T}}

\newcommand{\TRA}{\mathcal{T}_{\mathsf{RAs}}}
\newcommand{\TBox}[1]{\mathcal{T}_{#1}}
\newcommand{\SigTBox}[1]{\widetilde{\mathcal{T}}_{#1}}

\newcommand{\ClassifiedConcept}{\Sigma_{\mathsf{Classified}}}
\newcommand{\Known}{\mathsf{K}}
\newcommand{\Assertion}{\mathsf{A}}
\newcommand{\Possible}{\mathsf{P}}
\newcommand{\PossiblePairC}[1]{\mathsf{P}|_{#1}}
\newcommand{\RelationC}[2]{\mathsf{#1}|_{#2}}
\newcommand{\ToTest}{\mathsf{ToTest}}
\newcommand{\buildPre}[2]{\mathsf{buildPreModel}(#1:#2)}
\newcommand{\Prune}{\mathsf{prune}(\Possible, \Ont, A, \Known)}

\newcommand{\PossibleSubsumers}{\mathsf{possibleSubsumers}(\ConOnt{\Ont}, \Ont, E, s, \Assertion, \Known)}
\newcommand{\KnownSubsumers}{\mathsf{knownSubsumers}(\ConOnt{\Ont}, s_{0}, \Assertion)}
\newcommand{\laModel}{\mathcal{L}_{\Assertion}(s)}
\newcommand{\reach}[3]{#1 \rightsquigarrow_{#2} #3}
\newcommand{\hermitHierarchy}[4]{\mathsf{hierarchy}(#1, #2, #3, #4) }

\newcommand{\haoruoChange}[1]{{\color{blue}#1}}

\newcommand{\CTRG}[1]{\mathsf{CTRG}(#1)}
\newcommand{\CTRGP}[1]{\mathsf{CTRGP}(#1)}
\newcommand{\CTK}[1]{\mathsf{CTK}(#1)}

\newcommand{\STRG}[1]{\mathsf{STRG}(#1)}
\newcommand{\STRGP}[1]{\mathsf{STRGP}(#1)}

\newcommand{\Ptime}{\textsc{PTime}}
\newcommand{\machineOne}{\mathsf{Mach}_1}
\newcommand{\machineTwo}{\mathsf{Mach}_2}
\newcommand{\PelletAvoider}{\mathsf{PelletAvoider}}
\newcommand{\HermiTAvoider}{\mathsf{HermiTAvoider}}

\newcommand{\CheckedC}{\mathsf{CheckedC}}
\newcommand{\PotentialC}{\mathsf{PoC}}
\newcommand{\PositiveBenifit}[2]{\mathsf{PositiveBenifit(#1, #2)}}
\newcommand{\NegativeBenifit}[2]{\mathsf{NegativeBenifit(#1, #2)}}
\newcommand{\TotalOrderH}{\mathcal{H_T}}
\newcommand{\ParTOrder}{\mathcal{S_{\mathsf{TotalO}}}}

\newcommand{\ADTopDown}{\mathsf{AD aware Top Down Phase}}
\newcommand{\ADBottomUp}{\mathsf{AD aware Bottom Up Phase}}

\newcommand{\Equiv}[1]{\mathsf{Equiv}(#1)}
\newcommand{\SuperC}[1]{\mathsf{SuperC}(#1)}
\newcommand{\notSuperC}[1]{\mathsf{notSuperC}(#1)}
\newcommand{\HighestNotSuperC}[1]{\mathsf{HighestNotSuperC}(#1)}

\newcommand{\HighestClassifiedC}[1]{\mathsf{HighestClassifiedC}(#1)}

\newcommand{\LowestCon}[1]{\mathsf{LowestCon}(#1)}
\newcommand{\HighestCon}[1]{\mathsf{HighestCon}(#1)}
\newcommand{\PositBenifit}[1]{\mathsf{PositFreeSTs}(#1)}
\newcommand{\NegatBenifit}[1]{\mathsf{FreeNonS}(#1)}
\newcommand{\PotenNegatBenifit}{\mathsf{PFNS}}

\newcommand{\PotentialSubsumers}[1]{\mathsf{PotSubs}(#1)}
\newcommand{\LowerPotentialSubsumers}[1]{\mathsf{LowerPotSubs}(#1)}

\newcommand{\ConceptNameWithTopBot}[1]{\mathsf{C}^+\! (#1)}

\newcommand{\ADFNST}[1]{\mathsf{ADFNST}(#1)}

\newcommand{\MarkedC}{\mathsf{ClassifiedC}}
\newcommand{\SubA}[1]{\mathsf{SubConDirect(#1)}}
\newcommand{\SupA}[1]{\mathsf{SupConDirect(#1)}}
\newcommand{\ADTotalOrder}{\textsf{ADOrder}}

\newcommand{\subuniI}[1]{\mathsf{SubsI}(#1)}

\newcommand{\temporP}{\textcolor{blue}{Reasoner:Pellet}}

\newcommand{\PBFBotAxiom}{\varphi^{\bot} (\alpha)}
\newcommand{\PBFBotaxiom}[1]{\varphi^{\bot} (#1)}
\newcommand{\PBFTopAxiom}{\varphi^{\top} (\alpha)}
\newcommand{\PBFTopaxiom}[1]{\varphi^{\top} (#1)}
\newcommand{\PBFDletaAxiom}{\varphi^{\delta} (\alpha)}

\newcommand{\PBFBotPolarityForm}[2]{\varphi^{\bot}_{#1} ({#2})}  

\newcommand{\PBFTopPolarityForm}[2]{\varphi^{\top}_{#1} ({#2})}  

\newcommand{\PBFDeltaPolarityForm}[2]{\varphi^{\delta}_{#1} ({#2})}  

\newcommand{\ConSBot}{\textbf{Con}^{\bot}(\Sigma)}
\newcommand{\ConSTop}{\textbf{Con}^{\top}(\Sigma)}

\newcommand{\PBFDletaAxiomLength}{\lvert\varphi^{\delta} (\alpha)\rvert}

\newcommand{\PBFDletaAxiomSig}{\varphi^{\delta}_{\Sigma}(\alpha)}

\newcommand{\SigPBFDleta}{\mathcal{S}^{\varphi}_{\alpha}}

\newcommand{\PBFSubClassBot}[2]{\PBFBotPolarityForm{+}{#1} \land \PBFBotPolarityForm{-}{#2}}
\newcommand{\PBFSubClassTop}[2]{\PBFTopPolarityForm{+}{#1} \land \PBFTopPolarityForm{-}{#2}}

\newcommand{\PBFDleta}[1]{\varphi^{\delta} (#1)}

\newcommand{\PBFDeltaWholeFormula}{ \bigvee\limits_{\alpha \in \ensuremath{\mathfrak{a}}} \varphi^{\delta} (\alpha)}
\newcommand{\PBFDeltaWholeFormulaEvaluation}{ \bigvee\limits_{\alpha \in \ensuremath{\mathfrak{a}}} \varphi^{\delta}_\Sigma (\alpha)}

\newcommand{\map}{\varphi^{\delta}(\cdot)}
\newcommand{\MSS}{\delta-\Sigma^{mss}_{\mathcal{M}}}
\newcommand{\MSSs}{\textbf{mssig}(\Module, \Ont)}
\newcommand{\MSSsAtom}{\textbf{mssig}(\PricinpalAtom, \Ont)}

\newcommand{\PBFsAtomSig}[2]{\varphi_{#2}^{\delta} (#1)}

\newcommand{\PBFAD}{(\AD, \succeq, \map)}

\newcommand{\ChainsawLabel}[1]{\mathsf{L}(#1)}

\newcommand{\UnclassifiedAxioms}{\mathsf{UnclassifiedAxioms}}
\newcommand{\UnclassifiedModuleAxioms}{\mathsf{UnclassifiedModuleAxioms}}


\newcommand{\cps}{\ensuremath{\mathsf{CPs}}\xspace}
\newcommand{\rss}{\ensuremath{\mathsf{RRS}}\xspace}
\newcommand{\rw}{\ensuremath{\mathsf{RW}}\xspace}
\newcommand{\ma}{\ensuremath{\mathsf{MA}}\xspace}
\newcommand{\rma}{\ensuremath{\mathsf{RMA}}\xspace}
\newcommand{\width}{\ensuremath{\mathsf{Width}}\xspace}
\newcommand{\Cor}{\ensuremath{\mathsf{Cor}}\xspace}

\newcommand{\soas}{\mathfrak{A}(\Ont)}
\newcommand{\xsoas}{\mathfrak{A}^x(\Ont)}
\newcommand{\Omc}{\ensuremath{\mathcal{O}}\xspace}
\newcommand{\tbstar}{\top\!\bot\!^*}

\newcommand{\ClasiTime}[1]{\mathsf{CTH}(#1) }
\newcommand{\ClasiR}[2]{\mathsf{CT_{#2}}(#1)}

\newcommand{\CompTime}[1]{\mathsf{T}(#1) }

\newcommand{\DupliPer}{$\mathsf{DupliPer}$\xspace}
\newcommand{\ModPer}{$\mathsf{EL\text{-}ModPer}$\xspace}

\newcommand{\snomed}[1]{\mathsf{SNOMED}_{#1}}
\newcommand{\ana}{\mathsf{Anatomy}}
\newcommand{\gene}{\mathsf{GO}}
\newcommand{\hpo}{\mathsf{HPO}}

\newcommand{\structureConcept}{\mathsf{StructureConcept}}
\newcommand{\entireConcept}{\mathsf{EntireConcept}}

\newcommand{\modified}{}
\newcommand{\modifiedR}{}
\newcommand{\TODO}{{\color{blue}TODO}}

\newcommand{\PRE}{\mathsf{Preprocessing}}

\newcommand{\snomedCT}{\mathsf{SNOMED \: CT}}
\newcommand{\go}{\mathsf{GO}}
\newcommand{\Galen}{\mathsf{Galen}}
\newcommand{\fma}{\mathsf{FMA}}

\newcommand{\cnl}{\mathsf{CNL}}

\newcommand{\targetOnt}{\Ont_{\mathsf{Target}}}

\newcommand{\amazonNovaLit}{\textsf{amazon-nova-lite-v1}}
\newcommand{\amazonNovaMicro}{\textsf{amazon-nova-micro-v1}}
\newcommand{\deepseekThirtyOne}{\textsf{deepseek-chat-v3.1}}
\newcommand{\deepseekThirtyTwo}{\textsf{deepseek-v3.2-exp}}
\newcommand{\geminiTwoFiveLite}{\textsf{google-gemini-2.5-flash-lite}}
\newcommand{\geminiTwoFive}{\textsf{google-gemini-2.5-flash}}
\newcommand{\openAIFourZero}{\textsf{openai-gpt-4o-mini}}
\newcommand{\openAINanoFive}{ \textsf{openai-gpt-5-nano} }
\newcommand{\openAIFive}{\textsf{openai-gpt-5}}
\newcommand{\openAIoss}{\textsf{openai-gpt-oss-120b}}
\newcommand{\qwenPlus}{\textsf{qwen-plus}}
\newcommand{\qwenTurbo}{\textsf{qwen-turbo}}
\newcommand{\zAI}{\textsf{z-ai-glm-4.5-air}}

\newcommand{\amazon}{\mathsf{Amazon Nova}}
\newcommand{\DS}{\mathsf{Deepseek}}
\newcommand{\OAI}{\mathsf{OpenAI}}
\newcommand{\QWEN}{\mathsf{QWEN}}
\newcommand{\Gemini}{\mathsf{Google Gemini}}
\newcommand{\glm}{\mathsf{GLM}}

\newcommand{\goodColor}{\textcolor{red}}
\newcommand{\badColor}{\textcolor{blue}}

%% file: Introduction.tex
\section{Introduction}
\label{section:introduction}
The Web Ontology Language (OWL)~\cite{bechhofer2004owl,CHM+08} is a The World Wide Web Consortium (W3C) standard for representing ontologies with well-defined semantics. Over the past decades, OWL ontologies have had substantial impact across scientific and industrial settings. Biomedical ontologies such as $\snomedCT$~\cite{schulz2009snomed}, the Gene Ontology ($\go$)~\cite{gene2019gene}, the Human Phenotype Ontology ($\hpo$)~\cite{talapova2023human}, and the Foundational Model of Anatomy ($\fma$) ontology~\cite{rosse2003reference} demonstrate how long-term iterative development can yield highly mature resources that play a critical role in numerous large-scale applications. Encouraged by these developments, an increasing number of emerging domains have begun to develop their own ontologies, often within large collaborative initiatives. Notable examples include the Fluently ontology for Human–Robot Collaboration (HRC)~\cite{hall2024towards}, the Smart Applications REFerence (SAREF) ontology in the energy sector~\footnote{\url{https://saref.etsi.org/core/v4.1.1/}}, and the IES4 ontology~\footnote{\url{https://github.com/dstl/IES4}}
 developed within UK transport programmes. 

OWL is based on Description Logics (DLs)~\cite{DLIntro17}, which provide a well-defined semantics together with sound and complete reasoning. The formal semantics of DLs are non-trivial, and many domains lack established ontology engineering practices. In settings involving extensive modelling tasks, it is difficult to assume that domain experts fully interpret the semantics of OWL axioms in all cases. Ontology engineers, however, may lack detailed domain knowledge. The effectiveness of logical reasoning depends on the quality of the ontology. Interaction between ontology engineers and domain experts is therefore an important component of ontology modelling. Active learning~\cite{angluin1992computational} describes the interaction between teachers and learners through queries. In the context of ontology learning, this interaction can be formalised using Angluin’s exact learning framework~\cite{angluin1988queries} to model the roles of ontology engineers and domain experts~\cite{konev2018exact}. Using domain experts as teachers may be costly, since exact learning typically requires many queries. Large Language Models (LLMs) have therefore been considered as alternative teachers~\cite{magnini2025actively}, providing approximate responses based on their training data. This substitution, however, departs from the assumption in~\cite{konev2018exact} that the teacher answers queries truthfully and may introduce errors due to LLM hallucinations~\cite{ji2023survey}.

In this work, we propose LLM-assisted active learning approach which introduces LLMs as a third component within Angluin’s exact learning framework. The interaction between ontology engineers and domain experts remains unchanged. LLMs serve only to provide auxiliary guidance to ontology engineers, aimed at improving interaction efficiency. For each candidate axiom proposed by the ontology engineer, inspired by the reduction from subsumption testing to satisfiability, we construct the corresponding counter-concept and verbalise it using controlled natural language (CNL). The resulting description is used to formulate a prompt for the LLM, thereby reducing the reasoning task to a purely linguistic form. Using a confusion matrix, we show that under the Open World Assumption (OWA), false negatives and false positives have asymmetric impact. We formally prove that our method produces only Type II errors, which in the worst case delay ontology construction rather than alter the logical consequences of the ontology by introducing incorrect axioms. In our empirical study, recall is used to measure the frequency of Type II errors. Results across 13 commercial LLMs indicate that recall remains high even under conservative assumptions, suggesting that the approach can be applied in practical ontology modelling settings.

%% file: BackgroundKnowledge.tex
\section{Related Work}
\label{section:BK}

\subsection{OWL Ontology and Description Logics}

We assume that readers are familiar with DLs. DLs are a family of knowledge representation languages and are decidable fragments of first-order logic~\cite{DLIntro17}. DLs form the logical foundation of OWL. In DLs, concept names (e.g., Person, Pet) and role names (e.g., loves, has) are used to represent unary and binary predicates, respectively. A datatype theory $\mathcal{D}$ consists of a mapping from datatypes to sets of values, along with a function that maps each data value to its denotation within the corresponding value set. In this paper, we use $\Ont$ to denote a $\shiqD$~\cite{DLIntro17} ontology, and $\Sigma$ to denote a signature, i.e., a set of concept names $\ConceptName$, role names $\RoleName$, and datatype (or concrete) role names. We consider only $\shiqD$ ontologies consisting of Terminological Part (TBox). A role in $\shiqD$ is either a role name $r$ or its inverse $r^{-}$. The TBox of a DL ontology is a finite set of axioms consisting of General Concept Inclusions (GCIs) of the form $C \sqsubseteq D$, where $C,D$ are (possible complex) concepts. 

DLs have well-defined semantics. An interpretation $\mathcal{I}$ is a pair $\mathcal{I} = ( \Delta^{\mathcal{I}}, \cdot^{\mathcal{I}} ) $ with a non-empty set (domain of the interpretation) $\Delta^{\mathcal{I}}$ and the interpretation function $\cdot^{\mathcal{I}}$ which maps every concept name $A \in \ConceptName$ to a set $A^{\mathcal{I}} \subseteq \Delta^{\mathcal{I}}$, and every role name $r \in \RoleName$ to a binary relation $r^{\mathcal{I}} \subseteq \Delta^{\mathcal{I}} \times \Delta^{\mathcal{I}}$. The interpretation $C^{\mathcal{I}}$ of $C$ is defined inductively in~\cite{DLIntro17}. An interpretation $\mathcal{I}$ satisfies $C$ if $C^{\mathcal{I}}$ is not empty.
It satisfies the GCI $C \sqsubseteq D$ if $C^{\mathcal{I}} \subseteq D^{\mathcal{I}}$. If an interpretation $\mathcal{I}$  satisfies all GCIs in $\Ont$, then $\mathcal{I}$ is a model of $\Ont$. In this paper, checking whether $C$ is subsumed by $D$ with respect to the ontology $\Ont$, written as $\Ont \models^{?} C \sqsubseteq D$, is called a Subsumption Test (ST). ST is typically solved by reducing it to concept satisfiability, which is handled by tableau-based algorithms~\cite{HoKS06,Kaz08} implemented in modern DL reasoners~\cite{glimm2014hermit,sirin2007pellet,steigmiller2014konclude,zhao2021read}. That is, we construct a model $\mathcal{I}$ of $\Ont$, and check whether $\mathcal{I}$ satisfies $C \sqcap \neg D$. If no such model exists, then $\Ont \models C \sqsubseteq D$; otherwise, $\Ont \nvDash C \sqsubseteq D$. We refer to $C \sqcap \neg D$ as the counter-concept of the GCI $C \sqsubseteq D$.


\subsection{OWL Ontology Modelling Methods}
Several ontology modelling approaches focus on the interaction between ontology engineers and domain experts. CNLs~\cite{liang2011automatic,hart2008rabbit,fuchs2005attempto,cregan2007sydney} verbalise OWL axioms in natural language, reducing the need for domain experts to interpret formal DL semantics. Competency Questions (CQs) are used to elicit requirements and guide axiom formulation~\cite{ren2014towards}. Recent work has explored the derivation of CQs from functional requirements and their refinement using LLMs~\cite{alharbi25,alharbi2024investigating}. In contrast, our work does not rely on manually designed CQs. Instead, LLMs are used to generate counterexamples via CNL-based prompts inspired by the standard reduction from subsumption testing to satisfiability.




\subsection{Ontology Learning}
Ontology learning supports ontology modelling and can be characterised along several dimensions, including the data sources used, the type of knowledge produced, the learning methods applied (e.g. Formal Concept Analysis (FCA)~\cite{ganter1999formal,baader2006completing} and Association Rule Mining (ARM)~\cite{volker2011statistical}), the underlying semantics, the exploitation of available background knowledge, and the degree of domain expert involvement, which varies substantially across different approaches~\cite{sazonau2015general}.

\subsubsection{Induction} Inductive Logic Programming (ILP)~\cite{muggleton1994inductive} is the main supervised learning paradigm used for Class Description Learning (CDL)~\cite{baader2007computing,fanizzi2008dl,lehmann2010concept}, which aims to induce class expressions for a given class name from sets of positive and negative examples~\cite{lehmann2010concept}. Inspired by ILP, an unsupervised approach called General Terminology Induction (GTI)~\cite{sazonau2015general,sazonau2017mining} has been proposed for OWL ontologies under the OWA, aiming to learn sets of GCIs (hypotheses). Our work is highly inspired by GTI in that we fully consider unknown cases under the OWA and explicitly exploit the semantics of ontologies, rather than relying on predefined positive or negative examples. However, in contrast to GTI, we do not operate over entailed instance-level examples. We use LLMs as semantic assistants to support ontology modelling.

\subsubsection{Active Learning} Angluin’s exact learning framework~\cite{angluin1988queries} has recently been used to formalize the active learning~\cite{angluin1992computational} of lightweight OWL ontologies~\cite{konev2018exact,funk2019learning,funk2021actively}. In this framework, sets of GCIs (hypotheses) are learned through interactions between a learner and a teacher via two types of queries, namely MQs and equivalence queries (EQs). In a MQ, the learner proposes a candidate GCI and asks the teacher whether it is acceptable (e.g., entailed, consistent). The teacher responds with binary feedback. EQs are more challenging: the learner proposes a hypothesis consisting of a set of GCIs and asks the teacher whether it captures the target ontology. If the answer is negative, the teacher must explain why the hypothesis is incorrect, for example by indicating a GCI that is missing with respect to the target ontology.~\cite{ozaki2025actively}. To learn knowledge related to fruits, the learner may pose a MQ such as ``Is every apple a fruit?”, corresponding to the axiom $Apple \sqsubseteq Fruit$. The teacher then responds with a binary answer (Yes or No). An EQ, in contrast, corresponds to the learner claiming that the current set of axioms is sufficient to represent the domain knowledge about fruits. If this claim is rejected, the teacher provides evidence, such as an axiom or instance indicating that the current set of axioms is incomplete or incorrect. In our work, we focus on MQs and retain the interaction between the learner and the teacher. To reduce the burden on the teacher and the difficulty of interpreting ontology semantics, we introduce an LLM as a third component in the interaction.

\subsubsection{LLM-aware Ontology Learning}
In recent years, there has been increasing interest in exploring the interaction between LLMs and ontology reasoning tasks~\cite{he2023language,yang2026large}. Ontology learning methods have also made use of LLMs, for example in CDL~\cite{barua2025description,duranti2024llm} and in active learning~\cite{magnini2025actively}, where LLMs are used as teachers to answer MQs and EQs for learning lightweight OWL ontologies. In contrast, our work follows the same setting as in~\cite{konev2018exact}, where domain experts act as teachers, and introduces LLMs as an additional third component in the interaction explained in Section~\ref{section:LLMassistedAL}. Both our work and~\cite{magnini2025actively} use CNL to render MQs for LLMs. However, while~\cite{magnini2025actively} directly renders MQs, our approach instead provides the CNL representation of the corresponding counter-concept. Specifically, given the candidate axiom $Apple \sqsubseteq Fruit$, the CNL-based prompt in~\cite{magnini2025actively} is formulated as ``Can Apple be considered a subcategory of Fruit?”. Our method formulates the query as ``Find a real-world example of an individual that satisfies the following description: an individual that is an Apple and not a Fruit”. Furthermore, we adopt a confusion matrix~\cite{stehman1997selecting} to analyse the potential impact of LLM hallucinations on MQs.

%% file: Method.tex
\section{LLM-assisted Active Learning
}
\label{section:LLMassistedAL}
We now describe how our method relates to Angluin’s Exact Learning Framework~\cite{angluin1988queries}. Following the setting of~\cite{konev2018exact}, we use $\targetOnt$ to denote the target ontology; domain experts are treated as teachers, and ontology engineers as learners. Our approach also relies on the following assumptions concerning ontology engineers and domain experts:
\begin{itemize}
    \item The domain expert has knowledge of the domain but is unable to formalise the target ontology $\targetOnt$.
    \item The signature $\Sigma(\targetOnt)$ of the target ontology is shared with the ontology engineer, who otherwise has no knowledge of the domain.
\end{itemize}
In our work, we focus exclusively on MQs, which are used to check whether a candidate axiom belongs to the target ontology. EQs, such as testing whether the constructed ontology is equivalent to the target ontology, are left for future work. The assumptions related to MQs in~\cite{konev2018exact}, are described as follows:
\begin{itemize}
    \item The ontology engineer can pose queries of the form $\targetOnt \models^{?} C \sqsubseteq D$ to the domain expert, who answers truthfully.
\end{itemize}
Under our assumptions, ontology engineers lack domain knowledge and may therefore pose many MQs to domain experts, leading to an increased burden. In addition, the semantics of DLs are non-trivial, which makes it difficult to guarantee a correct interpretation of queries by domain experts. In our work, we retain these assumptions, but LLMs are used to complement background knowledge with ontology-aware examples, thereby reducing the burden on domain experts, as shown in Figure~\ref{fig:LLMAL}. LLMs are treated as a third component in the setting rather than as teachers, and the overall process remains human-centred. Final verification is performed exclusively by domain experts.

\begin{figure}[t]
\centering
\includegraphics[width=\columnwidth]{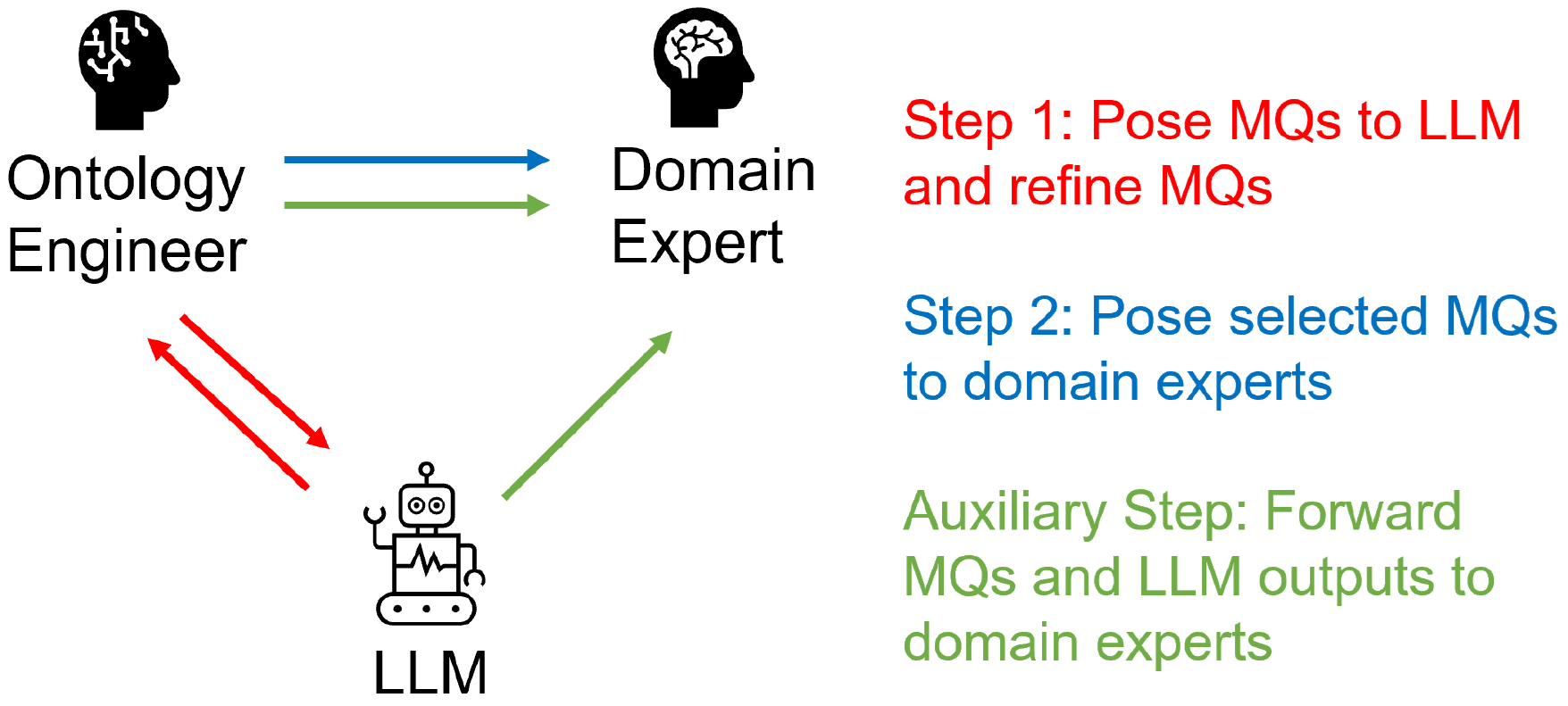}
\caption{\label{fig:LLMAL} An LLM-assisted Active Learning Setting}
\vspace{-1.5em}
\end{figure}


When $\targetOnt$ in MQs is replaced by a given ontology $\Ont$, a query of the form $\Ont \models^{?} C \sqsubseteq D $ is reduced to concept satisfiability by attempting to construct a model of $\Ont$ satisfying the counter-concept $C \sqcap \neg D$. If such a model exists, then $\Ont \nvDash C \sqsubseteq D$. Inspired by this reduction, our approach, referred to as the LLM Fuzzing Method, employs LLMs to describe approximate elements of a model of the target ontology $\targetOnt$ in natural language, drawing on existing real-world cases as discussed in Section~\ref{section:SAT}.

Based on this method, we construct an LLM-assisted active learning process for MQs. Rather than directly posing MQs to domain experts, ontology engineers first obtain natural language descriptions of real-world cases related to the counter-concept from LLMs. If these descriptions are sufficient to convince the ontology engineers, the corresponding MQs are not forwarded to domain experts, thereby increasing the involvement of ontology engineers in the process. If no real-world case can be identified for the counter-concept, the selected MQs are then posed to domain experts. In addition, as an auxiliary step, when ontology engineers remain uncertain about the LLM outputs, they may consult domain experts regarding the MQs together with the corresponding LLM-provided advice.

\section{LLM Fuzzing Method}
\label{section:SAT}
Concept satisfiability has high worst-case computational complexity (2NExpTime-complete for OWL 2~\cite{Kaz08}). Moreover, constructing an explicit model for a target ontology is itself highly challenging, particularly when it depends on collecting real-world data and when the target ontology $\targetOnt$ is difficult to describe. In our work, LLMs are used to provide approximate descriptions related to a model of the counter-concept.
We therefore propose an LLM Fuzzing Method, a symbolic–linguistic pipeline for ontology modelling, as shown in Figure~\ref{fig:method}. It is important to note that this method does not construct an exact model of the counter-concept, since LLMs are neither sound nor complete reasoners. Rather than asking LLMs to directly construct real-world models of $C \sqcap \neg D$, the counter-concept is translated into CNL prompts. This design preserves sound results prior to any interaction with the LLM, thereby reducing uncertainty in the overall process. Rather than relying on the LLM to simultaneously interpret OWL semantics and generate real-world cases, our approach decouples semantic reasoning from language generation. Specifically, we neither assume that domain experts can easily understand the formal semantics of OWL ontologies, nor expect LLMs to reason over them directly. As a result, the LLM is presented with a purely linguistic task, while soundness is ensured entirely within the symbolic layer, see Figure~\ref{fig:method}.
\begin{figure*}[tbp]
\centering
\includegraphics[width=\textwidth]{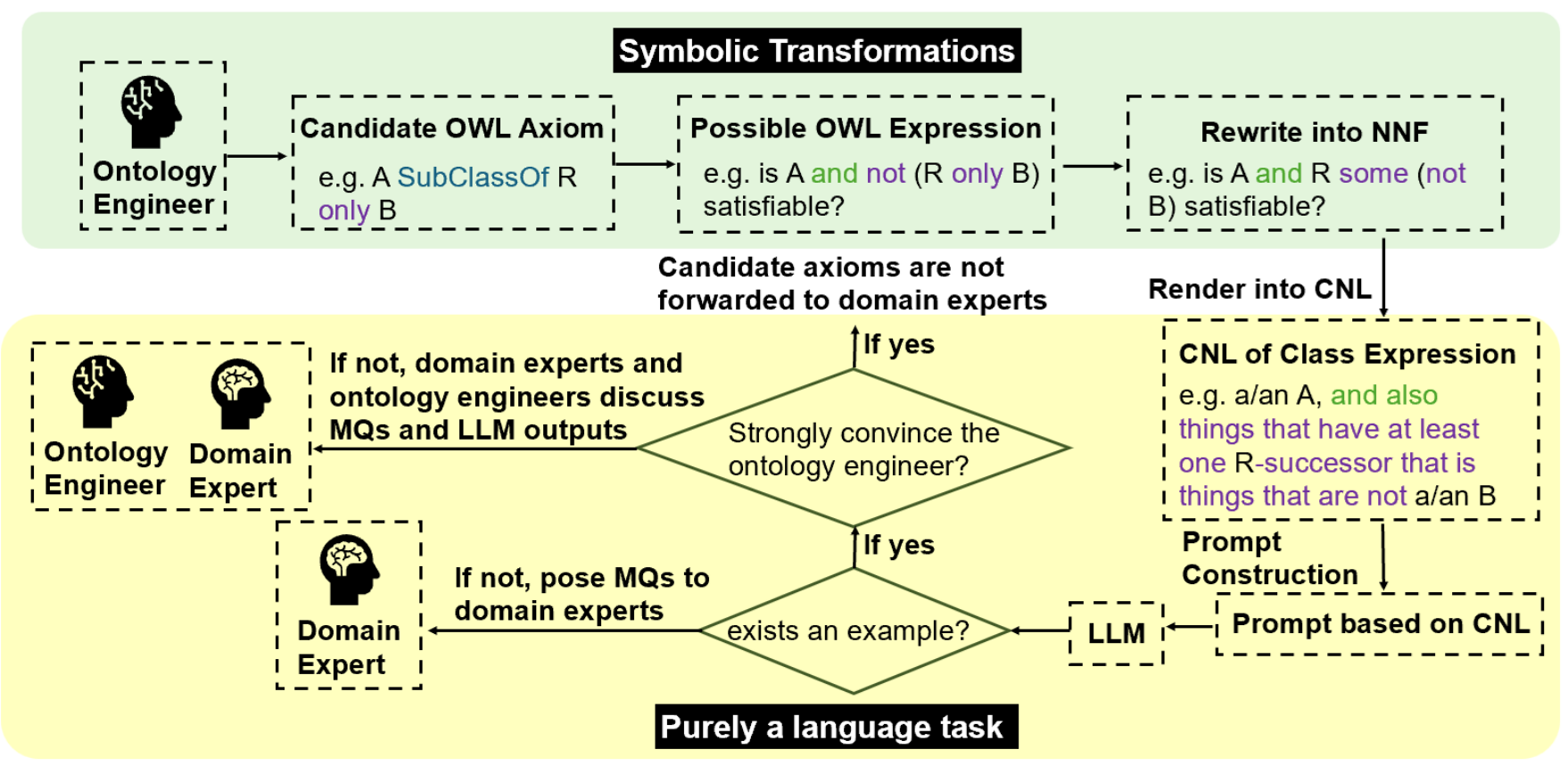}
\caption{\label{fig:method} LLM Fuzzing Method }
\vspace{-1.5em}
\end{figure*}

Since our method focuses on integrating LLMs into the active learning framework, it does not impose additional restrictions on the construction of candidate OWL axioms. The method can be treated as a plug-in component for active learning with MQs in OWL ontologies, as it preserves the interaction between ontology engineers and domain experts and follows all assumptions stated in~\cite{konev2018exact}. Candidate axioms can be manually constructed by ontology engineers, or derived from predefined positive and negative examples using GTI methods.

\subsection{Encoding $\shiqD$ MQs as CNL Prompts}
\begin{tcolorbox}[
breakable,
  colback=white,
  colframe=blue!75!black,
  coltitle=white,
  colbacktitle=blue!75!black,
  title=SAT-inspired Prompt,
  fonttitle=\bfseries,
  boxrule=0.8pt,
  left=6pt,
  right=6pt,
  top=6pt,
  bottom=6pt
]

\small
Find a real-world example of an individual that satisfies the following description:

\begin{quote}
\emph{An individual that is \(\cnl_{\mathcal{V}(F)}\).}
\end{quote}

At the very beginning of your answer, output only one word:
\begin{itemize}
  \item \textbf{Yes} — if a clear and natural real-world example exists.
  \item \textbf{No} — if no such real-world example exists, or if the description is contradictory or unrealistic.
\end{itemize}

After that, provide the analysis:
\begin{enumerate}
  \item If \textbf{Yes}:
    \begin{enumerate}
      \item Give the actual name or model of the individual.
      \item Provide a plain English explanation, grounded in real-world facts, explaining why the description is satisfied.
      \item If possible, include references or factual evidence.
    \end{enumerate}
  \item If \textbf{No}:
    \begin{itemize}
      \item Provide an analysis explaining why no real-world example exists, or why the description is contradictory or unrealistic.
    \end{itemize}
\end{enumerate}

\end{tcolorbox}

Given a candidate $\shiqD$ GCI of the form $C \sqsubseteq D$, posed as a MQ, we rewrite it into the corresponding counter-concept $C \sqcap \neg D$. To enable linguistic interaction while preserving symbolic semantics, we define a CNL for OWL class expressions with respect to $\shiqD$, inspired by existing work on CNLs~\cite{liang2011automatic}.

Following~\cite{horrocks2003reducing}, we use $\mathcal{V}(F)$ to denote the translation of an OWL fragment $F$ into the corresponding DL concept. OWL class expressions are normalized into negation normal form (NNF), as commonly adopted in Semantic Web toolchains, such as the OWL API~\cite{horridge2011owl}. The translations to $\shiqD$ and to CNL are summarized in Table~\ref{tab:owl-to-shiqd}, and the resulting CNL verbalization is denoted by $\cnl_{\mathcal{V}(F)}$. Using $\cnl_{\mathcal{V}(F)}$, we formulate a SAT-inspired prompt that transforms the original query into a purely linguistic task.

\begin{table*}[t]
\centering
\caption{Translation from OWL class and property expressions to $\shiqD$ and CNL, where $R$ may be an object property or its inverse. We use $\mathsf{annotation}(A)$ to denote the annotation of an OWL class. 
For readability, we abbreviate $\mathsf{restriction}$ as $\Res$ in this table.}
\small
\begin{tabular}{|l|l|l|}

\hline
\textbf{OWL fragment $F$} & \textbf{Translation $\mathcal{V}(F)$ } & $\cnl_{\mathcal{V}(F)}$ \\
\hline
$A$, OWL class name & $A$ & a/an $A$, (also known as $\mathsf{annotation}(A)$) \\
$B$, OWL datatype name & $B$ & datatype $B$\\
$R$, OWL object property name & $R$ & $R$\\
$T$, OWL datatype property name & $T$ & $T$ \\
$v$, OWL data value & $v$ & data value $v$ \\
\hline
$\mathsf{inverseOf}(R) $& $\mathcal{V}(R)^{-}$ & inverse of $\mathcal{V}(R)$\\ 
\hline
$\mathsf{intersectionOf}(C_1,\ldots,C_n)$ 
  & $\mathcal{V}(C_1) \cdots \sqcap \mathcal{V}(C_n)$ & $\mathcal{V}(C_1) \cdots$ , and also $\mathcal{V}(C_n)$ \\
$\mathsf{unionOf}(C_1,\ldots,C_n)$ 
  & $\mathcal{V}(C_1) \cdots \sqcup \mathcal{V}(C_n)$ & $\mathcal{V}(C_1) \cdots$, or also $\mathcal{V}(C_n)$\\
$\mathsf{complementOf}(C)$ 
  & $\neg \mathcal{V}(C)$ & things that are not $\mathcal{V}(C)$ \\
\hline
$\mathsf{\Res}(R\ r_1\ \ldots\ r_n)$
  & $\mathcal{V}(\mathsf{\Res}(R\ r_1))  \cdots \sqcap 
     \mathcal{V}(\mathsf{\Res}(R\ r_n))$ & $\mathcal{V}(\mathsf{\Res}(R\ r_1)) \cdots$ , and also
     $\mathcal{V}(\mathsf{\Res}(R\ r_n))$ \\
$\mathsf{\Res}(R\ \mathsf{allValuesFrom}(C))$
  & $\forall \mathcal{V}(R).\mathcal{V}(C)$ & things that have only $\mathcal{V}(R)$ -successors that are $\mathcal{V}(C)$\\
$\mathsf{\Res}(R\ \mathsf{someValuesFrom}(C))$
  & $\exists \mathcal{V}(R).\mathcal{V}(C)$ & things that have at least one $\mathcal{V}(R)$ -successor that is $\mathcal{V}(C)$\\
$\mathsf{\Res}(R\ \mathsf{minCardinality}(n \,C))$
  & $\geq n\,\mathcal{V}(R).\mathcal{V}(C)$ & things that have at least $ n \mathcal{V}(R)$ -successor that is $\mathcal{V}(C)$\\
$\mathsf{\Res}(R\ \mathsf{maxCardinality}( n \,C))$
  & $\leq n\,\mathcal{V}(R).\mathcal{V}(C)$ & things that have at most $ n \mathcal{V}(R)$ -successors that are $\mathcal{V}(C)$ \\
$\mathsf{\Res}(R\ \mathsf{cardinality}(n))$
  & $\geq n\,\mathcal{V}(R).\mathcal{V}(C) \sqcap \leq n\,\mathcal{V}(R).\mathcal{V}(C)$ &things that have exactly $ n \mathcal{V}(R)$ -successors that are $\mathcal{V}(C)$  \\
\hline
$\mathsf{\Res}(T\ r_1\ \ldots\ r_n)$
  & $\mathcal{V}(\mathsf{\Res}(T\ r_1))  \cdots \sqcap 
     \mathcal{V}(\mathsf{\Res}(T\ r_n))$ &$\mathcal{V}(\mathsf{\Res}(T\ r_1))  \cdots $, and also 
     $\mathcal{V}(\mathsf{\Res}(T\ r_n))$\\
$\mathsf{\Res}(T\ \mathsf{allValuesFrom}(D))$
  & $\forall \mathcal{V}(T).\mathcal{V}(D)$ & things that have only $\mathcal{V}(T)$ value of type $\mathcal{V}(D)$\\
$\mathsf{\Res}(T\ \mathsf{someValuesFrom}(D))$
  & $\exists \mathcal{V}(T).\mathcal{V}(D)$ & things that have at least one $\mathcal{V}(T)$ value of type $\mathcal{V}(D)$ \\
$\mathsf{\Res}(T\ \mathsf{minCardinality}(n\,D))$
  & $\geq n\,\mathcal{V}(T).\mathcal{V}(D)$ & things that have at least $n \mathcal{V}(T)$ value of type $\mathcal{V}(D)$\\
$\mathsf{\Res}(T\ \mathsf{maxCardinality}(n\,D))$
  & $\leq n\,\mathcal{V}(T).\mathcal{V}(D)$ & things that have at most $ n \mathcal{V}(T)$ value of type $\mathcal{V}(D)$\\
$\mathsf{\Res}(T\ \mathsf{cardinality}(n))$
  & $\geq n\,\mathcal{V}(T).\mathcal{V}(D) \sqcap \leq n\,\mathcal{V}(T).\mathcal{V}(D)$ &things that have exactly $ n \mathcal{V}(T)$ value of type $\mathcal{V}(D)$\\
  $\mathsf{\Res}(T\ \mathsf{value}(v))$ & $\exists \mathcal{V}(T).{\mathcal{V}(v)}$ & things that have $\mathcal{V}(D)$ equal to $\mathcal{V}(v)$ \\ 
  
\hline
\end{tabular}
\label{tab:owl-to-shiqd}
\end{table*}


Since $\cnl_{\mathcal{V}(F)}$ is a CNL, neither the prompt nor the expected response contains OWL or DL symbols.

\subsection{Asymmetric Impact of Type I and Type II Errors under the Open World Assumption}
The SAT-inspired prompt is used as input to the LLM. When LLMs are employed for direct ontology construction or treated as teachers in active learning, LLM hallucinations may introduce errors due to the lack of soundness, completeness, and explainability. In our method, the workflow is structured to analyse the effects of LLM hallucinations and to confine their impact to a controlled setting.

We analyse the possible outcomes of our approach using a confusion matrix~\cite{stehman1997selecting}. As ground-truth conditions, we distinguish whether the target ontology entails the candidate axiom, i.e., $\targetOnt \models C \sqsubseteq D$ (positive condition), or does not entail it, i.e., $\targetOnt \nvDash C \sqsubseteq D$ (negative condition). As predicted conditions, a \emph{Yes} response to the corresponding MQ is treated as a predicted positive, while a \emph{No} response is treated as a predicted negative. In our pipeline, a candidate axiom is accepted only if it is forwarded to and approved by the expert; it is rejected if the LLM strongly convinces the ontology engineer and the axiom is not forwarded for expert verification. Under this setting, we formally define Type I and Type II errors for MQs in active learning, together with the notion of Strong Conviction.
\begin{Definition}
A Type II error in active learning is defined as a predicted negative outcome (\emph{No}) for a MQ while the candidate axiom is entailed by the target ontology, i.e., $\targetOnt \models C \sqsubseteq D$.  \end{Definition}
A Type II error in active learning is equivalent to a false negative prediction.
\begin{Definition}
A Type I error in active learning is defined as a predicted positive outcome (\emph{Yes}) for a MQ while the candidate axiom is not entailed by the target ontology, i.e., $\targetOnt \nvDash C \sqsubseteq D$.
\end{Definition}

Type I and Type II errors are not equally severe during modelling the ontology due to the Open World Assumption (OWA). Under the OWA, non-entailment does not imply entailment of negation, and the \emph{unknown} case is an intended outcome of reasoning. As a consequence, the two types of errors have asymmetric effects. A Type I error, where an axiom that should not hold is incorrectly accepted, may introduce inconsistencies or conflicting reasoning results. In contrast, a Type II error, where an axiom that should hold is rejected, results only in an incomplete ontology. Under the OWA, such missing axioms do not introduce contradictions but merely reduce the completeness of the ontology or delay the ontology modelling. 

\begin{figure}[t]
\centering
\includegraphics[width=\columnwidth]{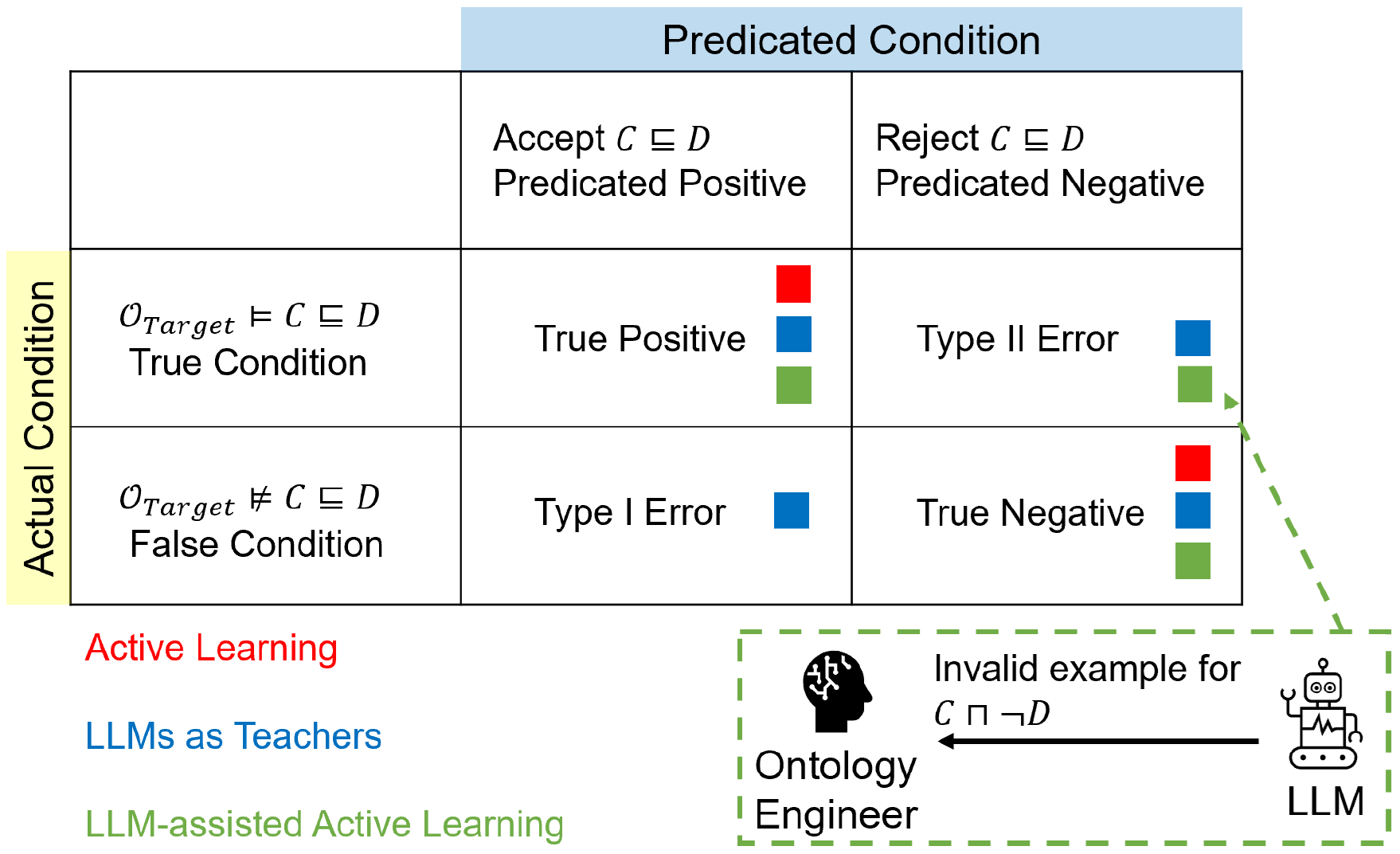}
\caption{\label{fig:confusionMatrix} Confusion Matrix for Active Learning }
\vspace{-1.5em}
\end{figure}

Figure~\ref{fig:confusionMatrix} illustrates the confusion matrix for active learning. In the setting described in~\cite{konev2018exact}, active learning does not incur Type I or Type II errors, since it is assumed that domain experts answer all MQs correctly. When LLMs are treated as teachers~\cite{magnini2025actively}, active learning may exhibit both Type I and Type II errors due to LLM hallucinations. In our approach, LLMs are treated as a third component rather than as teachers, and only Type II errors may be triggered. To establish this result, we first formalise the procedure of our method.
\begin{Definition}[Strong Conviction]
\label{definition:convice}
Given a positive LLM answer for the counter‑concept $C \sqcap \neg D$ accompanied by an explanation, 
we say that the LLM is \emph{strongly convincing} if the ontology engineer accepts that there exists a model $\mathcal{I}$ of $\targetOnt$ satisfying $(C \sqcap \neg D)$ and does not forward $C \sqsubseteq D$ to the expert.
\end{Definition}
\begin{corollary}
\label{coro:negativeConvinces}
If the LLM strongly convinces the ontology engineer,
then the predicted outcome of the MQ is negative.
\end{corollary}
\begin{proof}
By Definition~\ref{definition:convice}, strong conviction implies non-forwarding. By the decision rule, non-forwarding results in a predicted negative outcome and the rejection of the candidate axiom.
\end{proof}
Based on the above definitions and corollary, we establish the following theorem characterising the design property of our method. This property enables a principled trade-off between reducing the burden on domain experts through the use of LLMs and controlling the impact of LLM hallucinations on logic-based ontology construction.
\begin{theorem}
Under Assumption of correctness of domain experts~\cite{konev2018exact} ,
the proposed LLM-assisted active learning
can only incur Type II errors since every accepted axiom has been verified by a correct expert, no axiom with $\targetOnt \not\models C \sqsubseteq D$ can be accepted; thus Type I errors cannot occur.
\end{theorem}
\begin{proof}
Under Assumption, domain experts answer MQ correctly.
Hence, whenever a MQ is posed to an expert, its outcome coincides with the actual entailment status of the candidate axiom with respect to $\targetOnt$.

In the proposed framework, a candidate axiom $C \sqsubseteq D$ is not forwarded to the expert only if the LLM strongly convinces the ontology engineer.
In this case, the predicted outcome is negative and the axiom is rejected based on Corollary~\ref{coro:negativeConvinces}. If the rejection is correct, no error occurs. If $\targetOnt \models C \sqsubseteq D$ holds,
the rejection constitutes a Type II error. Since all positive outcomes require expert verification,
Type I errors are excluded.
Therefore, only Type II errors may arise.
\end{proof}

\subsection{Prompt Design for Considering Unknown Cases}
 A recent report by $\OAI$~\cite{kalai2025language} shows that the training and evaluation procedures of LLMs tend to reward guessing over acknowledging uncertainty. As a result, models are encouraged to produce confident answers rather than responding with ``unknown”, which can in turn lead to hallucinations. In our method, although the impact of LLM hallucinations is carefully considered and mitigated, the underlying training objectives of LLMs may still produce conservative yet extreme outputs that satisfy the formal constraints. For a MQ of the form $\targetOnt \models^{?} Apple \sqsubseteq Fruit$, We expect the LLM to respond negatively, indicating that no real-world case satisfies the counter-concept $Apple \sqcap \neg Fruit$. However, as illustrated in Figure~\ref{fig:appleExample}, LLMs such as ChatGPT~\footnote{\url{https://chatgpt.com/}}
 may return extreme yet semantically admissible examples that exploit lexical ambiguity to satisfy the given constraints.

\begin{figure}[t]
\centering
\includegraphics[width=\columnwidth]{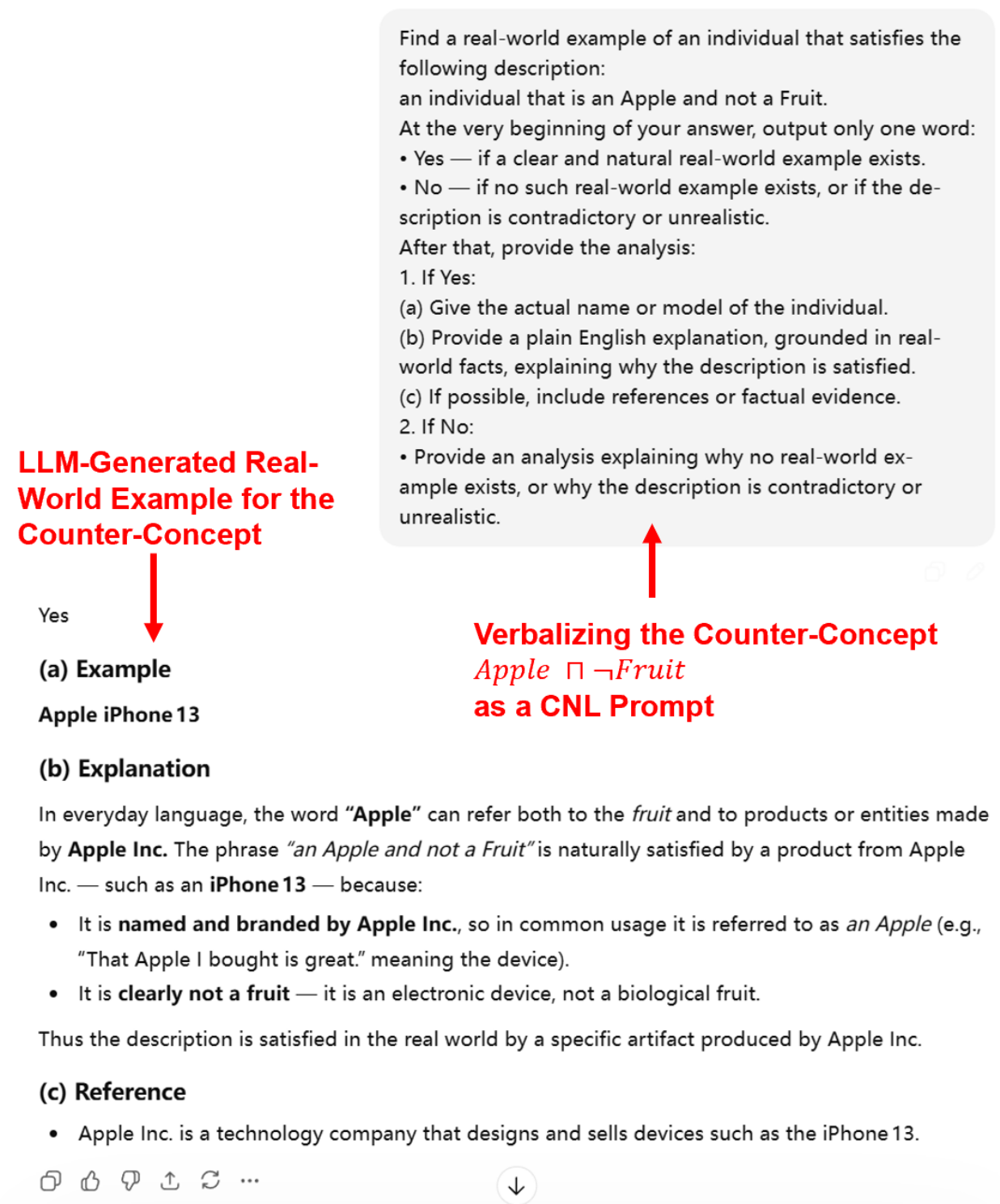}
\caption{\label{fig:appleExample} LLM-Generated Real-World Example for the Counter-Concept $Apple \sqcap \neg Fruit$.}
\vspace{-1.5em}
\end{figure}

From the perspective of background knowledge, examples generated by exploiting lexical ambiguity can be beneficial, as they expose semantic gaps that may otherwise remain implicit. In the example discussed above, such an extreme but admissible output makes the ontology engineer aware of a potential ambiguity between the lexical sense of Apple and the botanical notion of an apple, thereby prompting a more careful formulation of MQ.

At the same time, such examples introduce a risk. In domains where ontology engineers are less familiar with specialised domain knowledge, extreme but admissible outputs may appear plausible and lead to the rejection of candidate queries before they are forwarded to domain experts. In our framework, this does not introduce inconsistencies; it only delays ontology construction. Nevertheless, we introduce an alternative prompt formulation better aligned with the OWA and more suitable for complex domains. Specifically, we design an advanced variant of our prompt that explicitly emphasises the notion of ``unknown” in natural language, discouraging over-committed interpretations while preserving the construction of semantically admissible outputs.

\begin{tcolorbox}[
  breakable, 
  colback=white,
  colframe=blue!75!black,
  coltitle=white,
  colbacktitle=blue!75!black,
  title=Advanced SAT-inspired Prompt,
  fonttitle=\bfseries,
  boxrule=0.8pt,
  left=6pt,
  right=6pt,
  top=6pt,
  bottom=6pt
]

\small
\begin{quote}
\textbf{SAT-inspired Prompt}($\cnl_{\mathcal{V}(F)}$)
\end{quote}

Important rules:
\begin{itemize}
  \item Only provide a real-world example if it clearly and naturally satisfies the given description.
  \item If no real-world example exists, or if the description is contradictory or unrealistic, state this explicitly.
  \item When in doubt, it is preferable to explain why no suitable example exists rather than to construct an artificial one.
\end{itemize}

\end{tcolorbox}



%% file: Evaluation.tex
\section{Evaluation Methods and Setting}
\subsection{Research Questions}
Since our method can only produce Type II errors, recall is the most relevant metric for evaluating its performance, as defined below.
\[
\mathrm{Recall} = \frac{\mathrm{True\ Positive}}{\mathrm{True\ Positive} + \mathrm{False\ Negative}}
\]
In the evaluation, we treat axioms from existing, long-maintained ontologies as the gold standard. There are two assumptions in our experimental setting:
\begin{itemize}
    \item We assume that the existing axioms in these ontologies were proposed by ontology engineers and subsequently verified by domain experts.
    \item If the LLM provides a real-world example for the counter-concept, our framework assumes that this leads to strong conviction. This constitutes an extreme modelling assumption in which the ontology engineer does not prevent a Type II error, although, within our setting, consultation with domain experts remains possible.
\end{itemize}

To evaluate our approach, we treat these gold-standard axioms as candidate axioms and apply both the SAT-inspired prompt and its advanced version. The evaluation is conducted by analysing whether the LLM produces a real-world example for the corresponding counter-concept. If the LLM provides such an example, then, under the second assumption, strong conviction is triggered and the candidate axiom is not forwarded to the domain expert. Since, by construction, the axiom is entailed by the target ontology, this outcome is counted as a false negative. If the LLM does not provide such an example, the candidate axiom would be forwarded to the domain expert and correctly accepted. This case is therefore counted as a true positive.

In this section, we report on the empirical evaluation of our approach. More specifically, we aim to address the following Research Questions(RQ):
\begin{enumerate}[label=\textbf{RQ\arabic*.}, leftmargin=*, align=left]
\item What is the recall of our approach? This question helps us understand the extent to which LLM hallucinations lead to Type II errors.
\item Is there a difference in recall between the two prompt designs? This question examines whether encouraging unknown responses affects the performance of our approach.
\item What is the computational cost of the LLMs and of encoding ontologies into prompts? This question assesses the practical feasibility of applying our approach in industrial settings.
\end{enumerate}

\subsection{Setting}
Two widely used, long-term developed and maintained ontologies, $\go$ (version 2025-07-22) and $\fma$ (version 5.0.0) obtained from BioPortal\footnote{\url{www.bioportal.bioontology.org}}, are used as the experimental corpus in our study. Thirteen commercial LLMs are used in our experiments to evaluate the proposed method under the chosen metric. These models include \amazonNovaLit, \amazonNovaMicro, \deepseekThirtyOne, \deepseekThirtyTwo, \geminiTwoFiveLite, \geminiTwoFive, \openAIFourZero, \openAINanoFive, \openAIFive, \openAIoss, \qwenPlus, \qwenTurbo, \zAI. These models are drawn from six major LLM families: $\amazon$~\footnote{\url{www.aws.amazon.com/nova}}, $\DS$~\footnote{\url{www.deepseek.com}}, $\OAI$~\footnote{\url{www.openai.com}}, $\QWEN$~\footnote{\url{qwen.ai}}, $\Gemini$~\footnote{\url{gemini.google.com}}, $\glm$~\footnote{\url{chat.z.ai}}. The implementation of encoding $\shiqD$ ontologies into CNL prompts is based on the OWL API~\cite{horridge2011owl}, version~5.1.10.

All models were accessed via \textsf{OpenRouter} using default API parameters (e.g., default temperature and top-$p$), without task-specific tuning. This setup evaluates the zero-shot testing of the models rather than optimised performance for the specific task. Queries were submitted in batched mode rather than streaming. For each ontology, the first 1,000 axioms were selected, resulting in 2,000 candidate axioms in total due to computational cost constraints. Axioms were selected in file order. In future work, selection will be randomised and stratified by constructor profile to mitigate potential ordering bias.

\subsection{Results and Analysis}
The results are reported in Table~\ref{tab:runtime-recall-go-fma}. In the table, AvgT$_1$ and AvgT$_2$ denote the average response time per prompt when using the SAT-inspired prompt and the advanced SAT-inspired prompt, respectively. All response times are measured in seconds (s). AllT denotes the total time required by each LLM to process all 2000 prompt queries, measured in hours (h). Recall$_1$ and Recall$_2$ record the recall of our approach under the two prompt designs. Finally, Impro denotes the change in recall, computed as Recall$_2 -$ Recall$_1$.
\begin{table*}[t]
\centering
\small
\setlength{\tabcolsep}{6pt}
\renewcommand{\arraystretch}{1.1}
\caption{Runtime and recall comparison on GO and FMA ontologies. For each metric, the best result is highlighted in \goodColor{red}, while the worst result is highlighted in \badColor{blue}.}
\begin{tabular*}{0.85\textwidth}{@{\extracolsep{\fill}}lrrrrrl}
\toprule
\textbf{Model} 
& \textbf{AvgT$_1$ (s)} 
& \textbf{AvgT$_2$ (s)} 
& \textbf{AllT (h)} 
& \textbf{Recall$_1$} 
& \textbf{Recall$_2$} 
& \textbf{Impro} \\
\midrule
\multicolumn{7}{c}{$\underline\go$} \\
\addlinespace[0.2em]
\midrule
\amazonNovaLit & 3.62 & 3.57 & 2.00 & 85.4 & 88.8 & $+3.4$ \\
\amazonNovaMicro & \goodColor{2.22} & 2.14 & \goodColor{1.18} & \goodColor{99.9} & \goodColor{100} & $+0.1$ \\
\midrule
\deepseekThirtyOne & 13.49 & 12.01 & 7.08 & 56.4 & 76.4 & $+20$ \\
\deepseekThirtyTwo & 25.48 & 20.51 & 12.78 & 60.7 & 86.4 & \goodColor{$+25.7$} \\
\midrule
\geminiTwoFiveLite & \goodColor{2.22} & 2.28 & 1.25 & 88.3 & 89 & $+0.7$ \\
\geminiTwoFive & 3.70 & 3.62 & 2.03 & 69.7 & 76.2 & $+6.5$ \\
\midrule
\openAIFourZero & 2.38 & \goodColor{2.01} & 1.22 & \goodColor{99.9} & \goodColor{100} & $+0.1$ \\
\openAINanoFive & 29.11 & 29.05 & 16.16 & 70.7 & 75.8 & $+5.1$ \\
\openAIFive & \badColor{49.62} & \badColor{47.33} & \badColor{26.93} & 56.8 & 60.3 & $+3.5$ \\
\openAIoss & 15.67 & 15.19 & 8.57 & 90.2 & 89.9 & \badColor{$-0.3$} \\
\midrule
\qwenPlus & 10.17 & 10.17 & 5.65 & 73 & 75.1 & $+2.1$ \\
\qwenTurbo & 8.55 & 8.33 & 4.69 & \badColor{0.01} & \badColor{3.6} & $+3.5$ \\
\midrule
\zAI & 19.63 & 21.45 & 11.41 & 52.4 & 66.5 & $+14.1$ \\

\midrule
\multicolumn{7}{c}{$\underline

\fma$} \\
\addlinespace[0.2em]
\midrule
\amazonNovaLit & 2.86 & 2.80 & 1.57 & 98.3 & 98.9 & $+0.6$ \\
\amazonNovaMicro & 2.65 & 2.61 & 1.46 & 99.1 & 99.7 & $+0.6$ \\
\midrule
\deepseekThirtyOne & 11.18 & 11.02 & 6.17 & 92.2 & 96.6 & $+4.4$ \\
\deepseekThirtyTwo & 18.64 & 17.19 & 9.95 & 96.6 & 99.8 & $+3.2$ \\
\midrule
\geminiTwoFiveLite & \goodColor{2.51} & \goodColor{2.49} & \goodColor{1.39} & 99.3 & 99.1 & $-0.2$ \\
\geminiTwoFive & 3.41 & 3.51 & 1.92 & 92.7 & 96.3 & $+3.6$ \\
\midrule
\openAIFourZero & 3.95 & 4.03 & 2.22 & \goodColor{99.8} & \goodColor{99.9} & $+0.1$ \\
\openAINanoFive & 30.93 & 30.12 & 16.96 & 97 & 97.9 & $+0.9$ \\
\openAIFive & \badColor{41.15} & \badColor{37.82} & \badColor{21.93} & 67.9 & 75.4 & $+7.5$ \\
\openAIoss & 10.84 & 11.10 & 6.10 & 96.2 & 94.2 & \badColor{$-2$} \\
\midrule
\qwenPlus & 8.58 & 8.13 & 4.64 & 96.7 & 98.2 & $+1.5$ \\
\qwenTurbo & 7.70 & 7.38 & 4.19 & \badColor{0} & \badColor{1.4} & $+1.4$ \\
\midrule
\zAI & 14.64 & 14.66 & 8.14 & 73.8 & 81.9 & \goodColor{$+8.1$} \\
\bottomrule
\end{tabular*}
\label{tab:runtime-recall-go-fma}
\end{table*}

To answer \textbf{RQ1}, we observe that the majority of LLMs achieve strong recall performance. Specifically, among the 13 evaluated LLMs, 12 models attain recall values above 50\% on both the $\go$ and $\fma$ corpora, under both the SAT-inspired prompt and its advanced variant. Notably, \amazonNovaMicro\ and \openAIFourZero\ achieve recall rates exceeding 99\% across both ontologies and under both prompt settings. Nevertheless, there remain LLMs with poor performance. In particular, \qwenTurbo\ exhibits extremely low recall when using the SAT-inspired prompt, achieving 0\% and 0.01\% on the GO and FMA corpora, respectively. Even with the advanced prompt, its recall does not exceed 5\%, reaching only 3.6\% and 1.4\% on the two ontologies. Despite variability in individual model performance, the average recall across 13 LLMs under zero-shot evaluation (i.e., without task-specific fine-tuning) is 60.49\% for the SAT-inspired prompt and 76\% for its advanced variant on $\go$, and 85.35\% and 87.64\%, respectively, on $\fma$. These results indicate that, despite variability in individual model performance, our method demonstrates strong potential for leveraging the outputs of commercial LLMs as a supporting mechanism for ontology modelling.

We observe no consistent relationship between model size and recall performance in our evaluation. In particular, larger or more recent LLMs do not necessarily yield higher recall. For example, \openAIFive\ is both larger and more recent than \openAIFourZero, yet it exhibits substantially lower recall. The largest decrease is observed for Recall$_1$ on the $\go$ corpus, where \openAIFive\ performs 43.1\% worse than \openAIFourZero, while the smallest decrease occurs for Recall$_2$ on the $\fma$ corpus, with a reduction of 24.5\%. Given the characteristics of our method, models with weaker answering capabilities, such as earlier or less capable LLMs, may consistently respond that no counterexample can be found across queries, resulting in uniformly low recall. At the same time, there are also cases where larger and more capable models achieve markedly better recall. For instance, \qwenPlus, which is substantially larger than \qwenTurbo, shows a pronounced improvement in recall. On the $\go$ corpus, recall increases by more than 70\% under both prompt settings, while on the $\fma$ corpus the improvement exceeds 95\% for both prompts. It should be noted that this large relative gain is partly attributable to the extremely low baseline recall achieved by \qwenTurbo.

To answer \textbf{RQ2}, we observe that the advanced prompt improves recall for the majority of LLMs. The average improvement is 6.5\% on $\go$ and 2.28\% on $\fma$ across the 13 evaluated models. Specifically, 12 out of the 13 evaluated LLMs show improved recall on the GO corpus, while 11 out of 13 exhibit improvements on the FMA corpus when using the advanced prompt. In most cases, however, the magnitude of improvement is modest. Notable gains exceeding 10\% on the GO corpus are observed only for \deepseekThirtyOne, \deepseekThirtyTwo, and \zAI, with improvements of 20\%, 25.7\%, and 14.1\%, respectively. There are three cases where the improvement is negative: \openAIoss\ on both the $\go$ and $\fma$ corpora, and \geminiTwoFiveLite\ on the $\fma$ corpus. Among these, only \openAIoss\ on the $\fma$ corpus exhibits a decrease greater than 1\%, amounting to a 2\% reduction. Importantly, in all three cases the recall remains high, exceeding 85\%, indicating that the negative changes occur at already strong baseline performance levels. Overall, these results suggest that the advanced prompt design is beneficial in most settings and can serve as a useful reference for improving similar methods in future work.

To answer \textbf{RQ3}, we find that, overall, LLMs require relatively little time to respond to queries about whether a candidate axiom has a counterexample. Across all 13 models and both corpora, the average response time per query is under one minute for every model. Moreover, there is no clear difference in average response time between the SAT-inspired prompt and the advanced prompt, indicating that the additional prompt refinements do not introduce noticeable computational overhead on either corpus. Among the evaluated models, \openAIFive\ exhibits the longest average response time per query, reaching 49.62 and 47.33 seconds on the $\go$ corpus under the SAT-inspired and advanced prompts, respectively, and 41.15 and 37.82 seconds on the $\fma$ corpus under the two prompt settings. In contrast, \amazonNovaMicro, \geminiTwoFiveLite, and \openAIFourZero\ consistently achieve the shortest average response times in either AvgT$_1$ or AvgT$_2$. Notably, \openAIFourZero\ also attains the best recall performance on both corpora. This observation suggests that increased computational time does not necessarily lead to improved recall for our method on either the $\go$ or $\fma$ corpora.

\subsection{Limitation and Future Work}
In this paper, there are still several limitations that can be further considered in the future work. 
\subsubsection{Limited Coverage of $\shiqD$ Expressivity}
Although our method is designed for the DL $\shiqD$, the distribution of axioms in our experimental corpora is as follows. In $\go$, the selected axioms are of the form $A \equiv B \sqcap \exists R.C$. Such axioms can be decomposed into two GCIs, namely $A \sqsubseteq B \sqcap \exists R.C$ and $B \sqcap \exists R.C \sqsubseteq A$. Accordingly, 500 equivalence axioms were selected and rewritten into 1000 GCIs, which were then used in our method. In $\fma$, most of the selected axioms (986 out of 1000) are already in GCI form. The remaining 7 equivalence axioms were rewritten into two GCIs each. In the experiments, each GCI is treated as a separate candidate axiom.The 1000 selected axioms exhibit diverse syntactic forms, and some include datatype restrictions. However, the corpus lacks axioms involving inverse role and qualified number restrictions~\footnote{With respect to value restrictions, the construction of the counter-concept introduces universal restrictions. For a candidate axiom $A \sqsubseteq \exists R.C$, the counter-concept is $A \sqcap \forall R.(\neg C)$.}. In this study, the primary objective is to evaluate our method across multiple commercial LLMs, which incurs substantial commercial computational cost. As a result, we do not attempt to cover the full expressivity of $\shiqD$ by increasing the number of axioms. Future work will evaluate a larger ontology corpus using a locally deployed LLM in order to investigate expressivity-related aspects in greater depth.

\subsubsection{Application to Industrial Ontology Modelling}
Our method is specifically designed for ontology modelling. In this work, we design our own evaluation protocol, while using existing, long-maintained ontologies as the gold standard. Future work will involve ontology engineers and domain experts in real-world ontology development scenarios in order to assess the practical applicability of the proposed approach.

%% file: Conclusion.tex
\section{Conclusion}
Active learning provides a structured and principled approach to ontology learning, yet it may impose significant burden on domain experts, particularly given the non-trivial semantics of OWL. At the same time, while LLMs offer strong generative capabilities that can complement background knowledge, their lack of soundness and completeness makes their direct integration into logic-based ontology construction problematic. We propose an LLM-assisted active learning framework in which the LLM is introduced as a third component between ontology engineers and domain experts. Inspired by the reduction from subsumption testing to satisfiability in DLs, candidate axioms are reformulated into counter-concepts and verbalised via CNL. Under this design, LLM outputs can only induce Type II errors, thereby preventing the introduction of incorrect axioms that could lead to unintended entailments and erroneous reasoning outcomes in downstream tasks. Experiments on two long-maintained ontologies and thirteen commercial LLMs indicate that the observed recall levels correspond to a limited frequency of Type II errors for several models. Since such errors only delay ontology construction without affecting the correctness of derived entailments, the results support the feasibility of incorporating LLMs into active learning in a controlled way. This work outlines a principled pathway for integrating LLMs into ontology engineering in industrial applications, while balancing the effort required from ontology engineers and domain experts.